\documentclass{article}
\usepackage{spconf,amsmath,graphicx,epstopdf,amssymb,algorithm,algorithmic,float}
\usepackage{balance,amsthm,booktabs}

\newcommand\mydots{\hbox to 1em{.\hss.\hss.}}
\newtheorem{lemma}{Lemma}
\newtheorem{theorem}{Theorem}

\usepackage{amsmath}
\usepackage{amssymb}
\usepackage{amsthm}
\usepackage{latexsym}
\usepackage{verbatim}
\usepackage{graphicx}

\newtheorem{prop}{Proposition}

{\begin{list}               
    {$\bullet$ \hfill}{
        \setlength{\leftmargin}{\parindent}
        \setlength{\parsep}{0.04\baselineskip}
        \setlength{\itemsep}{0.5\parsep}
        \setlength{\labelwidth}{\leftmargin}
        \setlength{\labelsep}{0em}}
    }
{\end{list}}

\providecommand{\cref}[1]{Chapter~\ref{chap:#1}}

\providecommand{\R}{\ensuremath{\mathbb{R}}}

\providecommand{\norm}[1]{\left\lVert#1\right\rVert}
\providecommand{\inprod}[1]{\left\langle#1\right\rangle}
\providecommand{\set}[1]{\left\{#1\right\}}

\providecommand{\bydef}{\overset{\text{def}}{=}}

\providecommand{\rank}{\mathop{\mathrm{rank}}}

\renewcommand{\vec}[1]{\ensuremath{\mathbf{#1}}}
\providecommand{\mat}[1]{\ensuremath{\mathbf{#1}}}

\providecommand{\mA}{\mat{A}} 
 \providecommand{\mD}{\mat{D}}

\providecommand{\mN}{\mat{N}}
 \providecommand{\mP}{\mat{P}} 
 \providecommand{\mR}{\mat{R}}

\providecommand{\vc}{\vec{c}}

\providecommand{\vm}{\vec{m}} \providecommand{\vn}{\vec{n}} 
 \providecommand{\vp}{\vec{p}}
\providecommand{\vq}{\vec{q}} \providecommand{\vr}{\vec{r}}
\providecommand{\vs}{\vec{s}}

\providecommand{\vu}{\vec{u}} 

\providecommand{\vx}{\vec{x}}

 \providecommand{\vv}{\vec{v}}

\algsetup{indent=1em}

\ninept   

\title{Look, no Beacons! Optimal All-in-One EchoSLAM}

\name{Miranda Krekovi\'c$^{\,\dagger}$, Ivan Dokmani\'c$^{\,\ddagger}$, and Martin Vetterli$^{\,\dagger}$}
\address{\hspace{-4mm}
\begin{minipage}{.5\linewidth}
    \centering
    $^{\dagger}$
    School of Computer and Communication Sciences\\
    Ecole Polytechnique F\'ed\'erale de Lausanne (EPFL) \\CH-1015 Lausanne, Switzerland\\
    \{miranda.krekovic,martin.vetterli\}@epfl.ch
\end{minipage}%
\hspace{4mm}%
\begin{minipage}{.5\linewidth}
    \centering
    $^{\,\ddagger}$
    Institut Langevin\\
    CNRS, ESPCI Paris, PSL Research University\\
    1 rue Jussieu 75005 Paris, France\\
    ivan.dokmanic@espci.fr
\end{minipage}%
}

\newcommand\blfootnote[1]{%
	\begingroup
	\renewcommand\thefootnote{}\footnote{#1}%
	\addtocounter{footnote}{-1}%
	\endgroup
}
\begin{document}
%
\maketitle
\begin{abstract}
We study the problem of simultaneously reconstructing a polygonal room and a
trajectory of a device equipped with a (nearly) collocated omnidirectional
source and receiver. The device measures arrival times of echoes of pulses
emitted by the source and picked up by the receiver. No prior knowledge about
the device's trajectory is required. Most existing approaches addressing this
problem assume multiple sources or receivers, or they assume that some of
these are static, serving as beacons. Unlike earlier approaches, we take into
account the measurement noise and various constraints
on the geometry by formulating the solution as a minimizer of a cost function
similar to \emph{stress} in multidimensional scaling. We study uniqueness
of the reconstruction from first-order echoes, and we
show that in addition to the usual invariance to rigid motions, new
ambiguities arise for important classes of rooms and trajectories. We support our
theoretical developments with a number of numerical experiments.
\end{abstract}

\blfootnote{This work was supported by the Swiss National Science Foundation grant number 20FP-1 151073, ``Inverse Problems regularized by Sparsity''. ID was funded by LABEX WIFI (Laboratory of Excellence within the French Program “Investments for the Future”) under references ANR-10-LABX-24 and ANR-10-IDEX-0001-02 PSL* and by Agence Nationale de la Recherche under reference ANR-13-JS09-0001-01.}

\section{Introduction}
\label{sec:intro}

Autonomous mobile localization, and in particular simultaneous localization and mapping (SLAM),
has been an active topic of research for a long time. Different flavors of SLAM involve different sensing
modalities (for example, visual \cite{davison, blosch}, range-only \cite{blanco, djugash}, and acoustic SLAM \cite{hu}).
Almost all setups are defined as follows: given a sequence of
sensor observations and robot controls, estimate the robot's trajectory and
some representation of the environment---a map. Sensor measurements serve to improve
noisy kinematics-based trajectory estimates. 

We are interested in a more general problem in which the robot's kinematics is
a priori completely unknown, and we only obtain certain acoustic measurements at a few robot's
locations inside a room. While we often refer to acoustic echoes, any
range-based measurements will do---for instance reflections of ultra-wideband
signals \cite{meissner}. We assume no preinstalled infrastructure in the room, and only a bare
minimum on the robot---a single omnidirectional source and a single
omnidirectional receiver. This is different from our previous work where we
assumed some knowledge about the robot's trajectory \cite{krekovic}.

Our goal is twofold: 1) argue that multipath propagation conveys essential
information about the room geometry even with this rudimentary setup, and 2)
demonstrate in theory and experiment that this problem can be efficiently
solved with only a few measurements.
 
Multipath measurements have been used previously to do SLAM \cite{hu, meissner, gentner, dokmanic2},
and estimate room geometries \cite{dokthesis, dokmanic,
antonacci}. Most prior works use source and sensor arrays to do the estimation
\cite{ribeiro}. It is also common to assume either a fixed source or a fixed
receiver, so that the echoes correspond to virtual beacons from which we get
range measurements. In contrast, we assume no beacons, and we use a single omnidirectional source
and receiver. The same setup has been used before in \cite{peng} where the authors propose a
2D room reconstruction method based on noiseless times of arrival of the
first-order echoes. To cope with the everpresent measurement noise, we
formulate the solution as an optimization. 
 
Our contributions are as follows. We first provide the solution from noiseless
measurements based on simple trigonometry. We then use insights
from this part to formulate a cost criterion similar to the
well-known \emph{stress} function \cite{takane}, but adapted to the case of
collocated sensing. We show how this cost function can be restated in a
bilinear form for which efficient global solvers exist\footnote{While there
is no theoretical guarantee of worst-case polynomial-time complexity, the
runtime was consistently very short---in the milliseconds}; hence, we propose an algorithm that
always results with the best joint estimate of a room geometry and the
estimates of the measurement locations with respect to the mean square error.

The material is organized as follows. Section 2 introduces the notation and the problem setup. In Section 3 we
discuss the uniqueness of the mapping between the first-order echoes and
convex polyhedral rooms. A deterministic solution from the noiseless TOA
measurements and an algorithm that computes a globally optimal
solution from noisy TOA measurements are presented in Section 4. In Section 5
we numerically validate the performance of the algorithm.

\section{Problem setup}
\label{sec:problem_setup}

Suppose that a mobile device carrying an omnidirectional source and an
omnidirectional receiver traverses a trajectory described by the waypoints
$\set{\vr_i \in \R^3 \  : \ 1 \leq i \leq N}$. At every location, the source
produces a pulse, and the receiver registers the echoes. We assume that the
receiver can observe the first-order echoes.

Sound propagation is then described by a family of room impulse responses
(RIRs) where each RIR is idealized as a train of Dirac delta impulses produced
by the real and image sources, and recorded by the microphone at position
$\mathbf{r}_{i}$, 

\begin{equation}
h_{i}(t)=\sum_{j\geq 0} a_{j} \phi(t-\tau_{i,j})
\end{equation}
where $a_{j} $ are the received magnitudes that depend on the wall absorption coefficients
and the distance of the real or image source from the microphone.

To model echoes, we use the image source (IS) model \cite{allen, borish}. We
replace reflections with signals produced by image
sources---mirror images of the real sources across the corresponding walls
(Fig. \ref{fig:room}). We can describe the $j$th wall by its outward-pointing
unit normal $\vn_j$ and by some (any) wall point $\vp_j$. The wall then lies in the plane

\begin{equation}
{\cal P}_j =  \set{\vx \ : \ \vn_j^T (\vx - \vp_j) = 0 }
\end{equation}

With reference to Fig. \ref{fig:room}, the corresponding first-order image
source $\widetilde{\mathbf{s}}_{j}$ is computed as $\widetilde{\vs}_{j}=\vr+2 \inprod{ \vp_{j}-\vr, \vn_{j} } \vn_{j}$.
The propagation time $\tau_{i,j}$, also known as the time of arrival (TOA), is
proportional to the distance between the microphone $\vr_{i}$ and the
source $\widetilde{\vs}_{i,j}$:
\begin{equation}
\tau_{i,j} = \frac{\| {\widetilde{\vs}_{i,j}-\vr_{i}} \|}{c},
\end{equation}
where $c$ is the speed of sound.

\begin{figure}
\begin{center}
    \includegraphics[width=.4\linewidth]{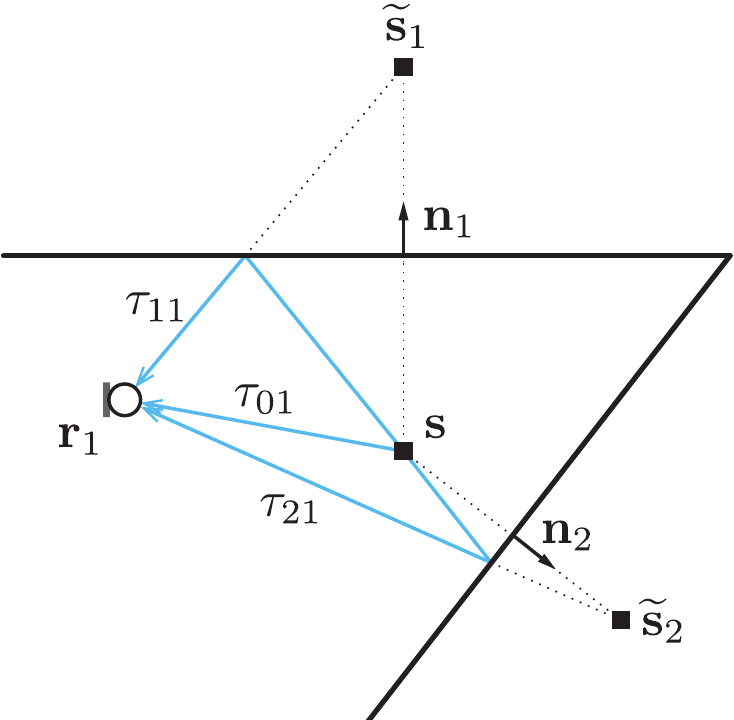}
    \vspace{-1em}
    \caption{Illustration of the image source model for first-order reflections in 2D.}
    \vspace{-1em}
    \label{fig:room}
\end{center}
\end{figure}

\begin{figure}
\begin{center}
    \includegraphics[width=7.6 cm]{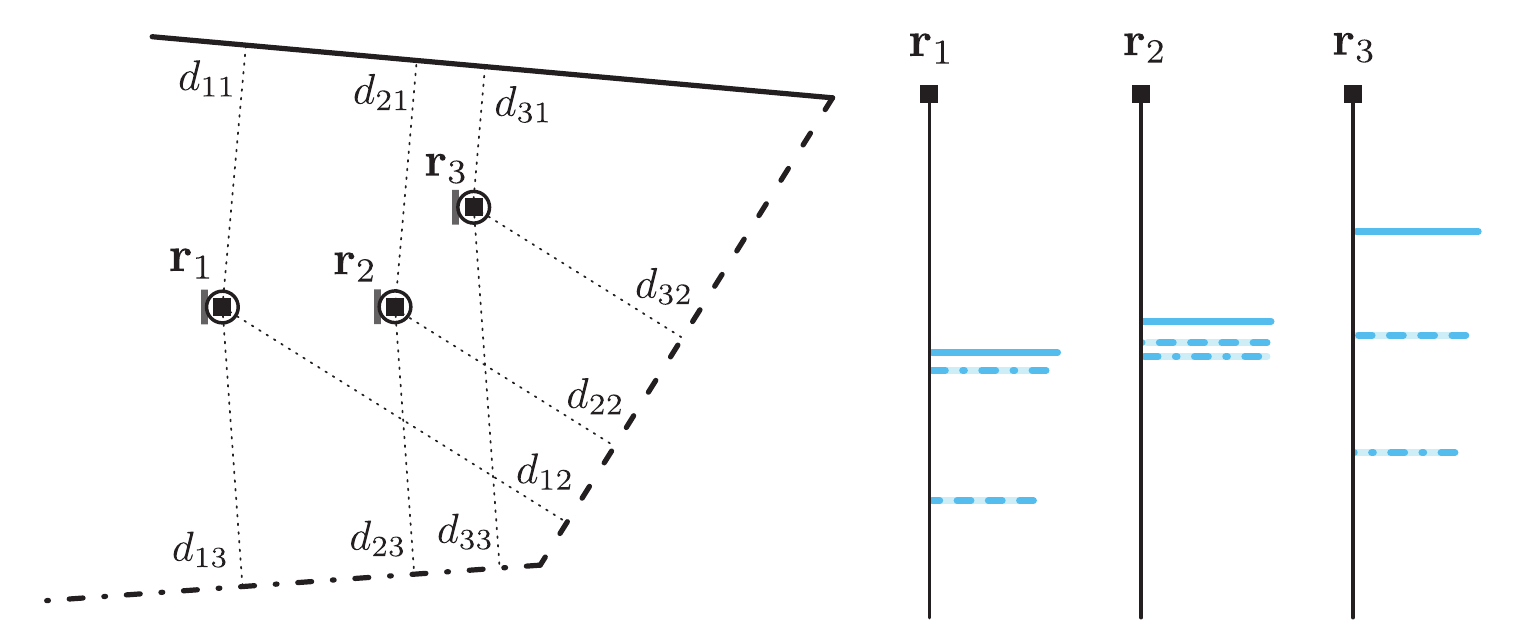}
    \vspace{-1em}
    \caption{\textit{Left}: Illustration of the first-order reflections in a 2D room for collocated microphone and source. \textit{Right}: Idealized room impulse responses consisting of the first-order pulses and recorded at $\vr_{1}$, $\vr_{2}$, and $\vr_{3}$. Corresponding walls and first-order echoes are marked same.}
    \vspace{-1em}    
    \label{fig:room_collocated}
\end{center}
\end{figure}

The case of collocated microphone and source is illustrated in Fig. \ref{fig:room_collocated}. A particularity of this setup is that the propagation times directly reveal the distances between the measurement locations and the walls, given by 
\begin{equation}
d_{i,j}=\frac{c \tau_{i,j}}{2}.
\end{equation}

In the following we assume that we have access to the first-order echoes and that the measurements of distance $d_{i,j}$ are obtained from their propagation times. Also, we assume to know correct labelling between each first-order echo and the corresponding wall.

\section{Characterization of rooms by first-order echoes}
\label{characterization}

We start by describing the solution with ideal, noiseless measurements. Recall
that we consider the room to be a convex polyhedron whose faces (reflectors)
are given in the Hessian normal form $\inprod{\vn_j, \vx } = q_j$,
where $q_j \geq 0$ is the distance of the plane from the origin, $q_j = \inprod{\vn_j, \vp_j}$.

\subsection{Linear dependence of propagation times}
\label{linear}

From the image source model we know that the distance$(\vr_{i}, {\cal P}_j)$ is
\begin{equation}
\label{eq:distdef}
d_{i,j} = \langle \vp_{j} - \vr_{i},  \vn_{j} \rangle
\end{equation}
for every measurement $i = 1, ...,N$ and wall $j = 1, ..., K$. We
choose $N$ so that $N \geq K$, and define $\mD  \in \mathbb{R}^ {N
\times K}$ to be the matrix with entries $d_{i,j}$. These times are not independent. In particular, we can state
the following result.

\begin{prop}
With $\mD$ defined as above, we have
\begin{equation}
\rank(\mD) \leq d + 1.
\end{equation}
\end{prop}

\begin{proof}
  Denoting $\mN = [\vn_1, \ldots, \vn_K]$, $\mP = [\vp_1, \ldots, \vp_K]$, and $\mR = [\vr_1, \ldots, \vr_N]$, we can write $\mD$ as

  \begin{equation}
    \mD = - \mR^T \mN + \vec{1} \vq^T.
  \end{equation}
  Because $\rank(\mR^T \mN)\leq d$ and
  $\rank(\vec{1} \vq^T) = 1$, the statement follows by the rank inequalities.
\end{proof}
\emph{Remark:} In most cases, $\rank(\mD)$ will be exactly $d+1$. It can be reduced only by special construction.

This property (or its approximate version in the noisy case) is useful for 1)
echo sorting, 2) in real situations when echoes will come in and out of existence,
and this property allows us to complete the matrix $\mD$ and estimate the
unobserved times.

\subsection{Uniqueness of first-order echoes}
\label{uniqueness}

At every one of $N$ locations we observe $K$ distance measurements defined in
\eqref{eq:distdef}. The unknowns that we are estimating are the wall
parameters---$\vp_{j}$ and $\vn_{j}$---and the measurement
locations, $\vr_{i}$. We are to recover $dK + dN$ unknown coordinates from $K N$ distance
measurements, where $d \in \set{2, 3}$ is the dimension of a space. It is
clear that translated and rotated version of the system will yield exactly the
same measurement, so we can fix these $d(d+1)/2$ degrees of freedom.
Requiring that $K N \geq d K + d N - d(d + 1) / 2$, we get a necessary condition for
the triples $(d, K, N)$ for which we can reveal the unknown
coordinates. Few examples of triplets are given in the Table \ref{table1}.

\begin{table}[H]
\centering
\begin{tabular}{@{} c c c c c c c c c @{} }
\toprule
 $d$ & &  & \bf 2 &  & \! &  & \bf 3 &  \\ 
 \midrule
 $K$ & & 3 & 4 & 5 & \! & 4 & 5 & 6 \\ 
 $N$ & & 3 & 3 & 3 & \! & 6 & 5 & 4 \\
\bottomrule
\end{tabular}
 \caption{Values of the smallest triplets $(d, K, N)$ sufficient for the room and the trajectory reconstruction.}
 \label{table1}
\end{table}

However, in this case counting the degrees of freedom is slightly misleading.
It turns out that it is impossible to always uniquely reconstruct the room and the
trajectory even when we fix the global translation and rotation, regardless of the number of measurements. There is
another invariance for a very important class of rooms---rectangular, which does not affect the distances
between the measurement locations and walls. Moreover, ambiguities arise for an important class of trajectories too---linear, as shown by the following lemma.

\begin{lemma} Two rooms with unit wall normals $\vn_{j}$ and $\vm_{j}$ generate identical measurements with trajectories $\vr_{i}$ and $\vs_{i}$ if and only if:
\[
\begin{bmatrix}
\vr_{1}^T  & -\vs_{1}^T    \\
\vr_{2}^T & -\vs_{2}^T  \\
 \vdots & \vdots  \\
\vr_{N}^T  & -\vs_{N}^T \\
\end{bmatrix} 
\begin{bmatrix}  
\vn_{1} & \vn_{2} & \hdots & \vn_{K} \\
\vm_{1} & \vm_{2} & \hdots & \vm_{K} 
\end{bmatrix} = \mathbf{0} \] 
\label{lemma1}
\end{lemma}

\begin{proof}
Let us assume that we have two sets of parameters that describe rooms and measurement locations, $\mathbf{\Phi}$ and  $\mathbf{\Phi}'$, which result in the same set of distance measurements. Sets consist of wall normals, wall points, and measurement locations: $\mathbf{\Phi}$ = $\big\{$ $\vn_j$, $\hdots$ $\vn_K$, $\vp_1$, $\hdots$, $\vp_K$, $\vr_1$, $\hdots$, $\vr_N$ $\big\}$ and $\mathbf{\Phi}'$ = $\big\{$ $\vm_j$, $\hdots$ $\vm_K$, $\vq_1$, $\hdots$, $\vq_K$, $\vs_1$, $\hdots$, $\vs_N$ $\big\}$. From the assumption of equal distance measurements, $d_{i,j}  = d'_{i,j} $, we obtain using \eqref{eq:distdef}
\begin{equation}
\big< \vp_j, \vn_j \big> - \big< \vr_i, \vn_j \big>  = \big< \vq_j, \vm_j \big> - \big< \vs_i, \vm_j \big> \quad \forall i, j.
\label{eq:main}
\end{equation}
A specific case of choosing $\vn_j, \vr_i, \vm_j, \vs_i$ such that
\begin{equation}
\big< \vr_i, \vn_j \big> =  \big< \vs_i, \vm_j \big>
\label{eq:main_simple}
\end{equation} 
implies that we need to find $\vp_j$ and $\vq_j$ such that
\begin{equation}
\big< \vp_j, \vn_j \big> = \big< \vq_j, \vm_j \big>.
\label{eq:side_simple}
\end{equation}
As for each wall $j$ we obtain one equation \eqref{eq:side_simple} with $2d$ variables that are independent across the walls, we can always find a solution of such linear equation. Therefore, we focus on analyzing the solutions of Eq. \eqref{eq:main_simple}. In a matrix form we have
\begin{equation}
\mathbf{R_0} \mathbf{N_0} = \mathbf{0},
\label{eq:rono}
\end{equation}
\[ \mathbf{R_0} =
\begin{bmatrix}
\vr_1 & \vr_2 & \hdots & \vr_N  \\
-\vs_1 & -\vs_2 & \hdots & -\vs_N  \\
\end{bmatrix} ^T, \] \[
\mathbf{N_0} = \begin{bmatrix}  
\vn_{1} & \vn_{2} & \hdots & \vn_{K} \\
\vm_{1} & \vm_{2} & \hdots & \vm_{K}  
\end{bmatrix}. \]
The solution exists when the columns of $\mathbf{N_0}$ are in the nullspace of $\mathbf{R_0}$ and the rows of $\mathbf{R_0}$ are in the nullspace of $\mathbf{N}_\mathbf{0}^T$.

In the general case when $ \big< \vr_i, \vn_j \big> \neq  \big< \vs_i , \vm_j \big>$ we denote the difference $\big< \vp_j, \vn_j \big>  -  \big< \vq_j, \vm_j \big> = \big< \vr_i, \vn_j \big> - \big< \vs_i , \vm_j \big> = w_{i,j}$ and obtain the equations:
\begin{align}
\big< \vr_i, \vn_j \big> = \big< \vs_i , \vm_j \big> + w_{i,j}  \label{eq:proof1_1} \\
\big< \vp_j, \vn_j \big> =  \big< \vq_j , \vm_j \big> + w_{i,j} \label{eq:proof1_2}
\end{align}
We notice that all variables in Eq. \eqref{eq:proof1_2} depend only on the wall index, so we require the same for a new variable, $w_{i,j} = w_{k,j} = w_j$, $\forall i, k$. Further, we can write any real number as an inner product of some vectors. We define $\mathbf{v}_j \in \mathbb{R}^2$, such that $w_j = \vv^T_j \vn_j$. Then, from Eq. \eqref{eq:proof1_1} we obtain $\big< \vr_i - \vv_j, \vn_j \big> = \big< \vs_i, \vm_j \big>$ that can be written in a matrix form as
\begin{equation}
\mathbf{R} \mathbf{N} = \mathbf{0},
\label{eq:rn}
\end{equation}
\[ \mathbf{R} =
\begingroup
\renewcommand*{\arraystretch}{1}
\begin{bmatrix}
\vr_{1} - \vv_{1} & \hspace{-2px} \mydots & \hspace{-2px} \vr_{1} - \vv_{K} & \hspace{-2px} \mydots & \hspace{-2px} \vr_{N} - \vv_{1} & \hspace{-2px} \mydots & \hspace{-2px} \vr_{N} - \vv_{K}  \\
-\vs_{1} & \hspace{-2px} \mydots & \hspace{-2px} -\vs_{1}  & \hspace{-2px} \mydots & \hspace{-2px} -\vs_{N} & \hspace{-2px} \mydots & \hspace{-2px} -\vs_{N} 
\end{bmatrix} \endgroup
, \]
\[\mathbf{N} = \begin{bmatrix}  
\vn_{1} & \vn_{2} & \hdots & \vn_{K} & \hdots & \vn_{1} & \vn_{2} & \hdots & \vn_{K} \\
\vm_{1} & \vm_{2} & \hdots & \vm_{K}  & \hdots & \vm_{1} & \vm_{2} & \hdots & \vm_{K}  
\end{bmatrix}. \]
Analogously to Eq. \eqref{eq:rono}, this implies that the rows of $\mathbf{R}$ must be in a nullspace of $\mathbf{N}^T$. The matrix $\mathbf{N}$ is constructed by $N$ times concatenating $2d \times K$ matrices $\mathbf{N_0}$. To find a solution of Eq. (\ref{eq:main}), we study the nullspace of the matrix $\mathbf{N}^T$ ($\mathbf{N}_\mathbf{0}^T$) and find the vectors $\vr_i$,  $\vv_j$ and $\vs_i$ that live in the nullspace. Generically for $K \geq 4$, the nullspace is empty. For it to be nonempty, we must explicitly assume different linear dependencies among the columns or rows. The detailed analysis is omitted due to space limit and is deferred to in our forthcoming publication \cite{krekovic_journal}. In the analysis we also show that the Eq. \eqref{eq:rono} gives the same characterization of the uniqueness property as the general case, Eq. \eqref{eq:rn}. In other words, $\mathbf{R_0} \mathbf{N_0} = \mathbf{0}$ covers all possible combinations of rooms and measurement locations that result in the same set of distance measurements.
\end{proof}
Remarkable is a consequence of Lemma \ref{lemma1}, stated in Theorem \ref{thm:third}.
\begin{theorem}
Regardless of the number of measurements, same set of the first-order echoes in two rooms can occur if and only if the room is a parallelogram or measurement locations are collinear.
\label{thm:third}
\end{theorem}

\noindent Profs are straightforward from the omitted analysis.

\section{Room geometry estimation and measurement localization}

\subsection{Noiseless measurements}

The room geometry and the measurement locations can be jointly recovered from
only a few noiseless distance measurements. A benefit of the idealized
measurements is to simply illustrate how the echoes reveal essential geometric
information about the sources--channel--receivers system. Moreover, the same
approach to the solution can be used when the noise is small, while the
globally optimal solution is introduced in Section \ref{minimization}.

In order to fix some degrees of freedom we choose 4 scalars required to
specify the translation and the rotation, $\vr_{1} = [0 \text{ } 0]^{T}$ and $\vn_{1} = [0 \text{ } 1]^{T}$. This results in new
identities $\langle \vp_{j}, \vn_{j} \rangle = d_{1,j}$ and
$r_{y,i} = d_{1,1} - d_{i,1}$, and simplifies the initial
formulation of the distance measurements \eqref{eq:distdef}. The system of
equations from which we jointly reveal $\vn_{j}$ and $\vr_{i} $ for all $i = 1, \ldots, N$ and $j = 1, \ldots, K$
thus consists of $(N-1)(K-1)$ equations of the form

\begin{equation}
\label{equations}
{r}_{x,i} {n}_{x,j} + (d_{1,1} - d_{i,1})  {n}_{y,j} = d_{1,j} - d_{i,j}.
\end{equation}

\subsection{Noisy measurements}
\label{minimization}

In reality we get noisy distance measurements and solving polynomial equations
might be problematic. Additionally, this algebraic approach makes it
challenging to incorporate any prior knowledge we might have about the room or
the trajectory. It is easy to imagine scenarios where \emph{some} inertial
information is available, and the above approach provides no simple way to integrate it. 

To address these shortcomings, we formulate the joint recovery as an
optimization problem. Noisy measurements are given as

\begin{equation}
 \widetilde{d}_{i,j} = d_{i,j} + \epsilon_{i,j} = \inprod{\vp_j - \vr_i, \vn_j} + \epsilon_{i,j},
\end{equation}
where $\epsilon_{i,j}$ is noise. It is natural to seek the best estimate of the unknown vectors by solving
\begin{align}
\label{eq:cost_function}
& \underset{
\begin{subarray}
\text{ \text{  } \text{  } } q_{j}, \vn_{j}, \vr_{i} \\ i \leq N, j \leq K 
\end{subarray}} {\text{minimize}}
& & \sum_{i = 1}^{N} \sum_{j=1}^{K} (\widetilde{d}_{i,j} - q_{j} + \vn_{j}^{T} \vr_{i})^2 \nonumber \\
& \text{subject to}
& & \norm{\vn_{j}} = 1,\ j = 1, ..., K.
\end{align}

We make two remarks. First, this cost function is analogous to the
familiar \emph{stress} \cite{takane} function common in multidimensional
scaling. Second, if $\epsilon_{i,j}$ are iid Gaussian random variables, then
solving \eqref{eq:cost_function} gives us the maximum likelihood estimates of
the trajectory and the room geometry.

The cost function is not covex and minimizing it is difficult due to many local minima.
However, different search methods have been developed that guarantee global
convergence in algorithms for nonlinear programming (NLP). We mention two
approaches: global optimization for bilinear programming and
interior-point filter line-search algorithm for large-scale nonlinear
programming (IPOPT).

The first approach is based on a formulation of our cost function as a bilinear program
(BLP). The bilinearization is achieved as rewriting the problem so
that all the nonconvexities are due to bilinear terms in the objective
function or the constraints. 

By introducing a new variable $u_{i,j} \bydef \vn_j^T \vr_i$, and letting
$\vu_{i,j} = [q_j, \ u_{i,j}]^T$, we can rewrite the cost function as
\begin{align*}
 &\text{minimize} &&\sum_{i = 1}^{N} \sum_{j=1}^{K} \ \vu_{i,j}^T \mA \vu_{i,j} - 2 d_{i,j} \vc^T \vu_{i,j} - d_{i,j}^2 \\
 &\text{subject to} &&\vn_j^T \vn_j = 1, j = 1, ..., K \\
 &&& u_{i,j} = \vn_j^T \vr_i, \ \vu_{i,j} = [q_j,\ u_{i,j}]^T,
\end{align*}
with 

\begin{equation}
  \mA = 
  \begin{bmatrix}
  1 & -1 \\ -1 & 1
  \end{bmatrix}
  \ , \ 
  \vc = 
  \begin{bmatrix}
  -1 \\ 1
  \end{bmatrix},
\end{equation}
resulting in a bilinear cost function subject to bilinear constraints. Although such formulation also belongs to the class of nonconvex nonlinear
programming problems with multiple local optima, (empirically) efficient strategies for its global minimization exist.
A linear programming (LP) relaxations suggested by McCormick \cite{cormick} are the most widely used techniques for
obtaining lower bounds for a factorable nonconvex program. In the numerical
experiments for this paper, instead of using our home-brewed implementation of
McCormick's method, we use the significantly faster line search filter method (IPOPT) proposed in
\cite{wachter1} and implemented in an open-source package as a part of ​COIN-OR Initiative \cite{coinor}. 
Further improvement of computational efficiency is the focus of an ongoing research. By resorting to IPOPT we also get a guarantee on global convergence
under appropriate (mild) assumptions. Although we do not yet have a formal proof that our problem satisfies
the required assumptions, exhaustive
numerical simulations suggest that the method efficiently finds an optimal
solution in all our test cases and demonstrates favorable performance compared
to other NLP solvers. Additional advantage of using this method is that we can add various
constraints without introducing auxiliary variables and increasing the size of
a model. 

\section{Experimental results}
\label{results}

For numerical simulations we used interface IPOPT through the optimization software APMonitor Modeling Language.

Figures \ref{fig:linear_nonunique} and \ref{fig:paral_nonunique} illustrate Theorem \ref{thm:third} 
and exhibit non-unique reconstructions of room geometries from the first-order echoes.
Room impulse responses recorded at the locations marked with the same color and number in different rooms are equivalent.
Corresponding echoes and walls are indicated by the same line pattern. The left sides of Figures \ref{fig:noise_room} and \ref{fig:noise_loc} respectively depict the reconstruction of the room geometry and measurement locations for a fixed number of measurements ($N = 8$) and Gaussian noise, $\mathcal{N}(0, 0.05^2)$. The right sides show the dependence of the estimation error on the standard deviation of the noise. This standard deviation was increasing from 0 to 0.15 with steps 0.005, and for each value we performed 1000 experiments. The average SNR is plotted above the graph.

\begin{figure*}[t]
 \vspace{-1em}
  \centering
  \centerline{\includegraphics[width=\textwidth]{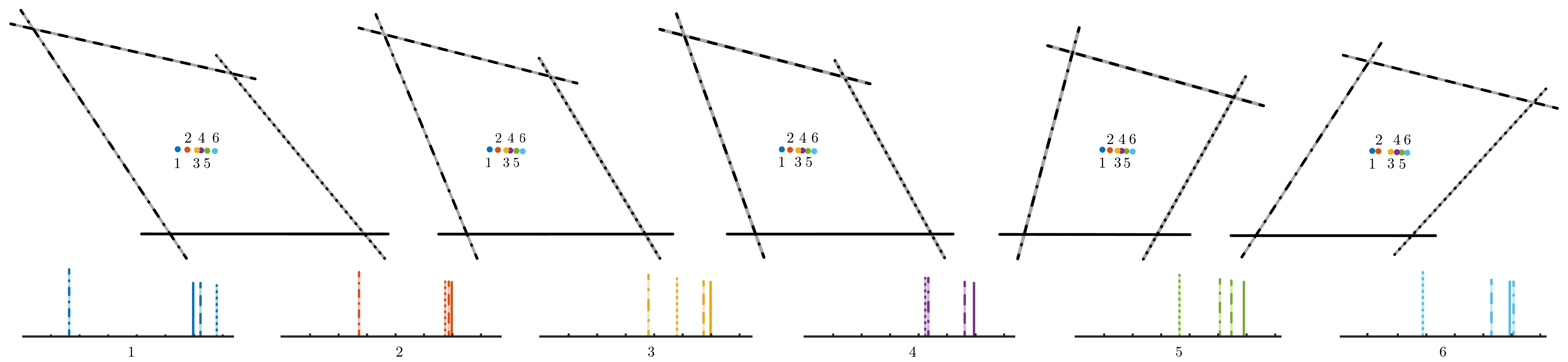}}
   \vspace{-1em}
   \caption{Example of rooms and collinear measurement locations that result in the same set of first-order echoes.}
   \label{fig:linear_nonunique}
\end{figure*}

\begin{figure*}[t]
  \centering
  \centerline{\includegraphics[width=\textwidth]{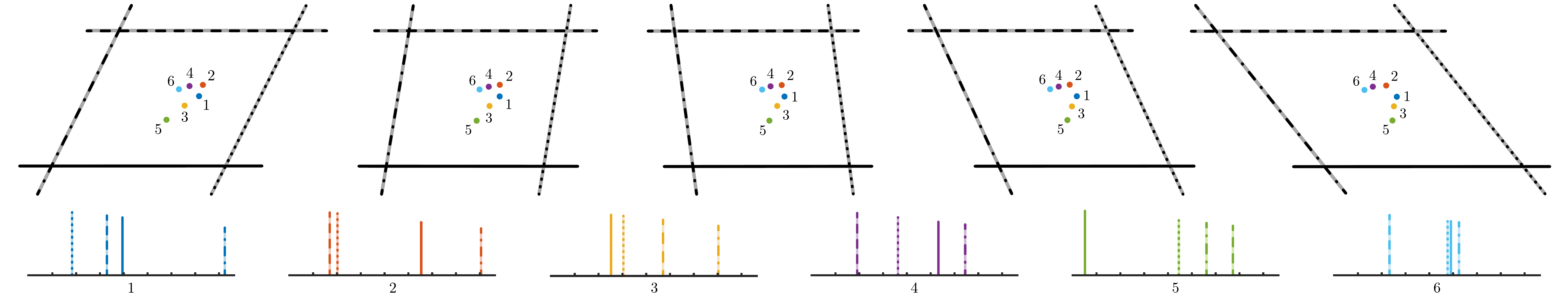}}
   \vspace{-1em}
   \caption{Example of parallelogram rooms and measurement locations that result in the same set of first-order echoes.}
  \label{fig:paral_nonunique}
\end{figure*}

 \begin{figure}[H]
 \vspace{-1em}
 \begin{minipage}[h]{0.4\linewidth}
  \centering
    \centerline{\includegraphics[width=3.3cm]{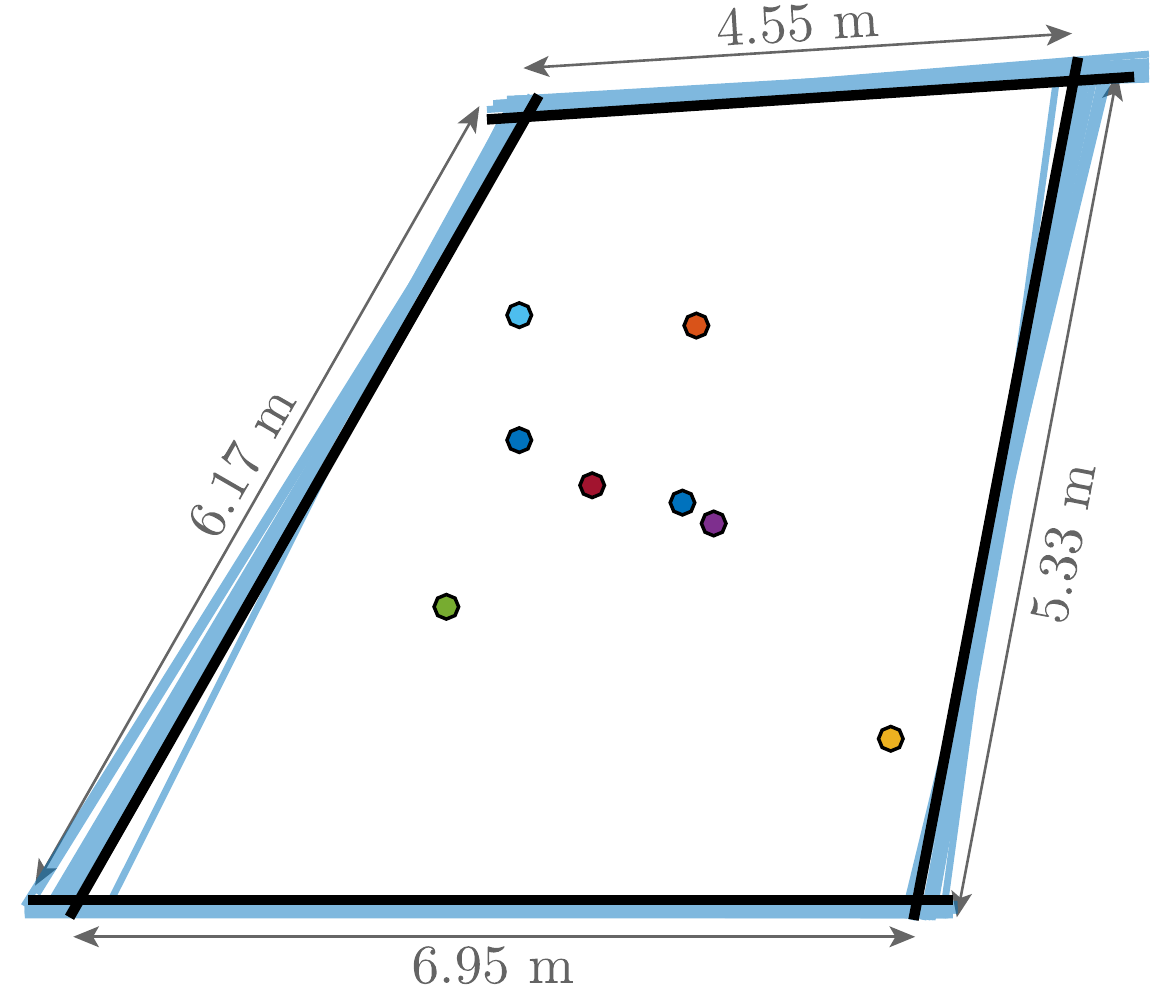}}
 \end{minipage}
 \begin{minipage}[h]{0.6\linewidth}
  \centering
  \centerline{\includegraphics[width=4.2cm]{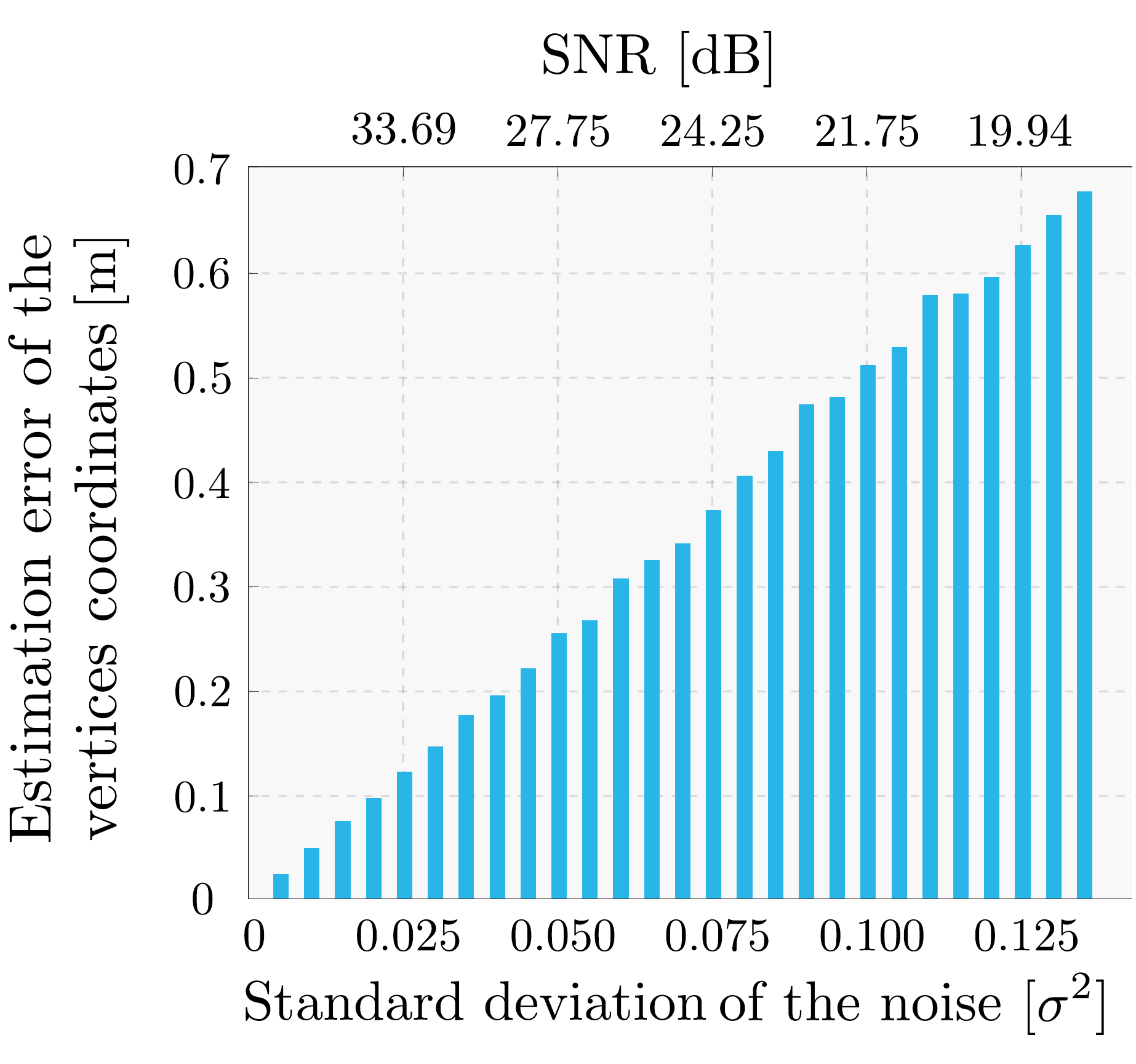}}
 \end{minipage}
  \vspace{-1em}
 \caption{\textit{Left:} Room reconstructions of 10 experiments with measurement noise $\epsilon_{i,j} \sim \mathcal{N}(0, 0.05^2)$ are illustrated in blue, while the original room is colored black.  \textit{Right:} Dependence of the estimation error on the noise. Estimation error is defined as the Euclidean distance between the vertices of the original and the reconstructed room.}
 \vspace{-1em}
 \label{fig:noise_room}
\end{figure}

 \begin{figure}[H]
 \vspace{-1em}
 \begin{minipage}[h]{0.4\linewidth}
  \centering
   \centerline{\includegraphics[width=3cm]{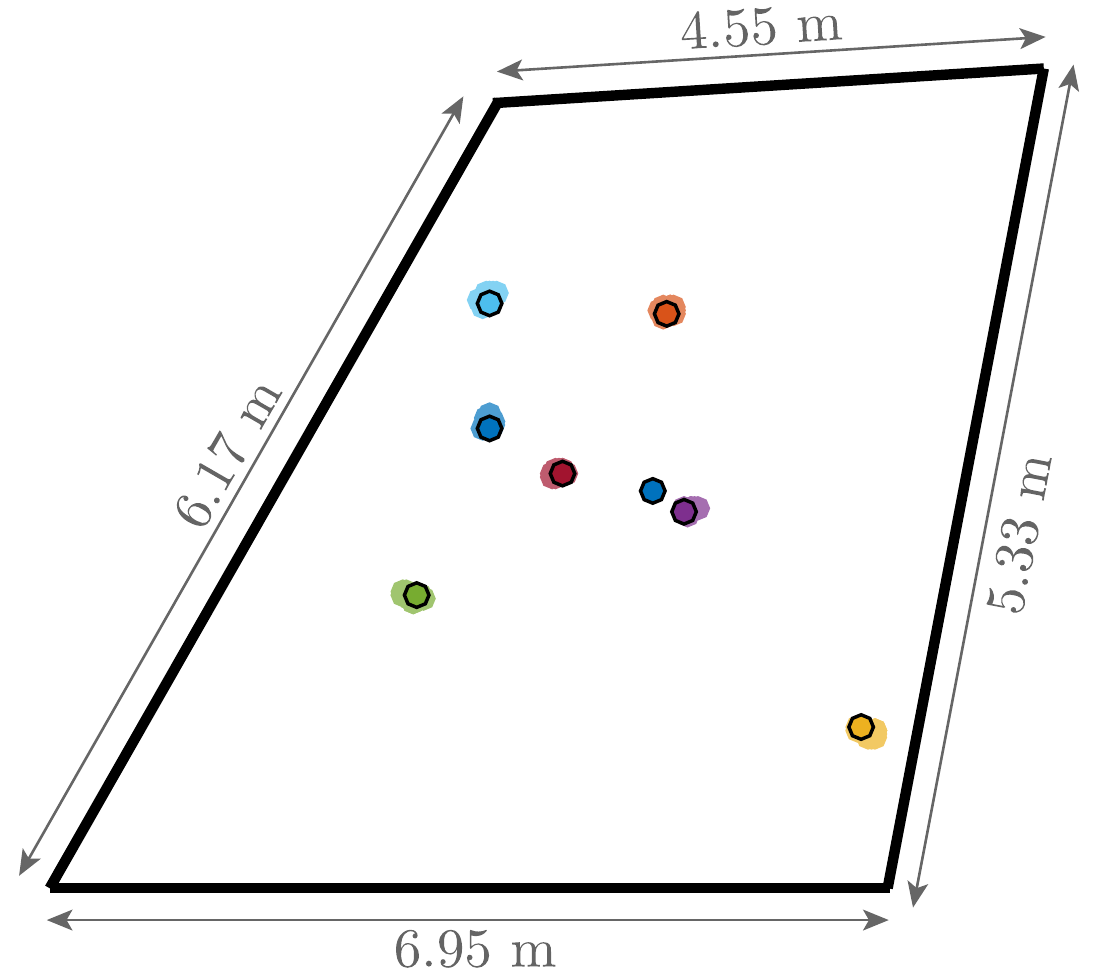}}
 \end{minipage}
 \begin{minipage}[h]{0.6\linewidth}
  \centering
  \centerline{\includegraphics[width=4.2cm]{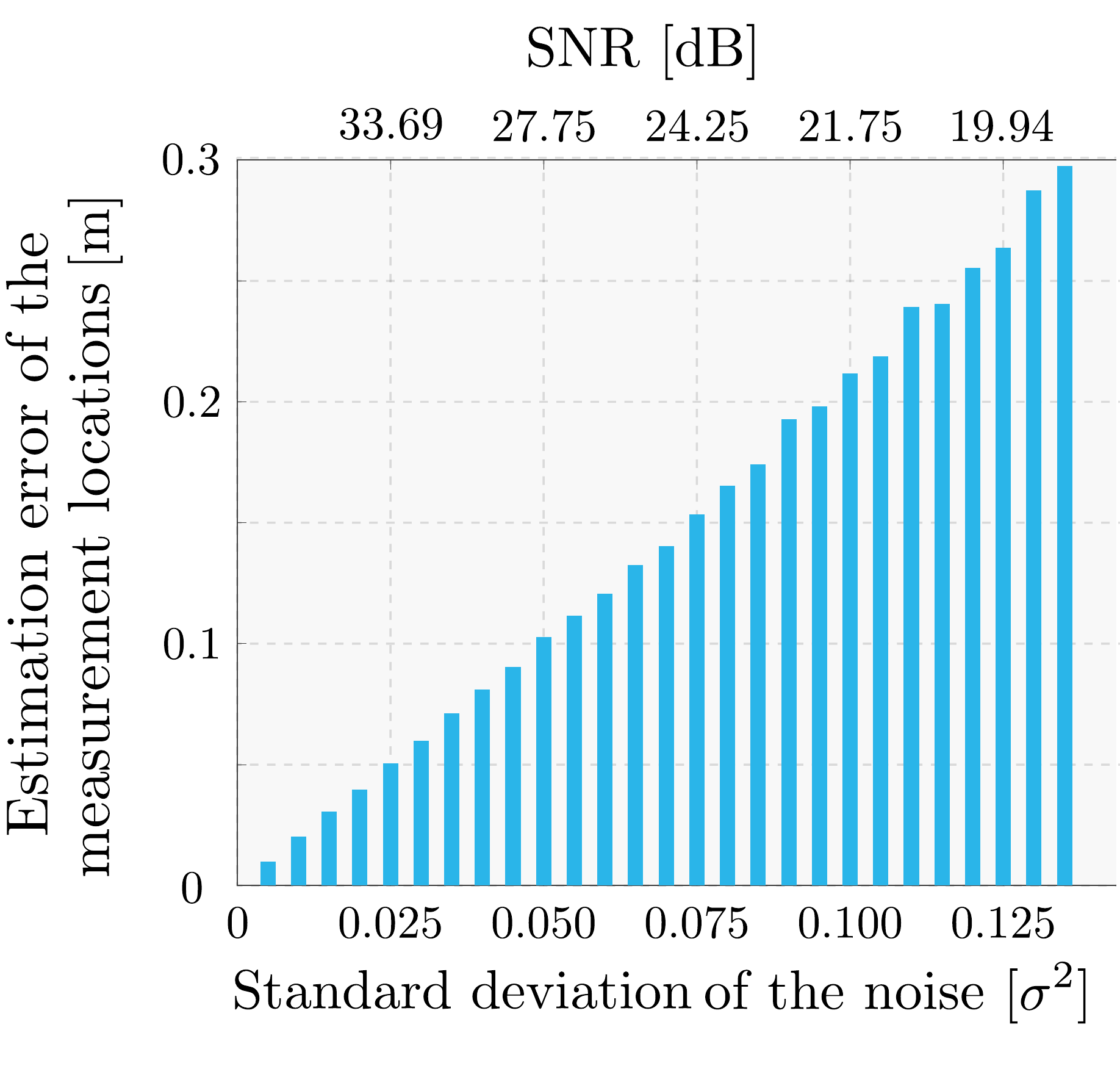}}
 \end{minipage}
  \vspace{-1em}
 \caption{\textit{Left:} Estimated measurement locations of 10 experiments are plotted in lighter shades, while the original measurement locations stand out as bordered circles. \textit{Right:} Dependence of the estimation error on the noise. The estimation error is defined as Euclidean distance between original and reconstructed measurement locations. }
   \vspace{-1em}
 \label{fig:noise_loc}
\end{figure}

\section{Conclusion}
\label{sec:conclusion}
We presented an algorithm for reconstructing the 2D geometry of a convex polyhedral room from a few first-order echoes. Our sensing setup is rudimentary---we assumed a single omnidirectional sound source and a single omnidirectional microphone collocated on the same device, and no preinstalled infrastructure in the room. We established conditions on room geometry and measurement locations under which the first-order echoes collected by a microphone define a room uniquely. Further, we stated our problem as a non-convex optimization problem and proposed a fast optimization tool which simultaneously estimates the geometry of a room and locations of the measurements. We empirically observe that the structure of our cost function admits efficient computation of the globally optimal solution. We showed through extensive numerical experiments that our method is robust to noise. 

Ongoing research includes the extension of the study to 3D and the verification of the method through experiments with real measured room impulse responses.

\bibliographystyle{IEEEbib}

\end{document}